\documentclass{article}

\usepackage{graphicx}
\usepackage[preprint]{corl_2026} 

\usepackage{color}
\usepackage{colortbl}

\usepackage{wrapfig}

\usepackage{CJKutf8}
\usepackage{amsmath}
\usepackage{amssymb}
\usepackage{mathtools}
\usepackage{amsthm}

\usepackage{pifont}

\newcommand{\cmark}{\ding{51}}
\newcommand{\xmark}{\ding{55}}

\newtheorem{proposition}{Proposition}

\newtheorem{assumption}{Assumption}
\newtheorem{definition}{Definition}

\definecolor{red_}{RGB}{250,0,0}
\definecolor{lightblue_}{RGB}{0,0,255}
\definecolor{darkgreen_}{RGB}{0,200,0}

\definecolor{darkblue_}{RGB}{0,0,139}
\definecolor{darkyellow_}{RGB}{204,153,0}

\long\def\comment#1{ }

\definecolor{tablerowgray}{gray}{0.92}

\usepackage{booktabs}
\usepackage{multirow}
\usepackage{graphicx}

\usepackage{amsmath}

\newcommand{\venue}[1]{{\scriptsize\textit{(#1)}}}

\title{CT-VAM: A Cerebello-Thalamic-Inspired Vision-Action Model for Efficient Visuomotor Control}

\author{
  Jiacheng Li\thanks{Equal contribution.} \quad
  Yize Guo\footnotemark[1] \quad
  Jiabin Guo\footnotemark[1] \quad
  Qingchen Liu \quad
  Jiahu Qin\thanks{Corresponding author.}\\
  AIRLab, Department of Automation\\
  University of Science and Technology of China\\
  Hefei, China
}

\begin{document}
\maketitle


\begin{abstract}
Vision-language-action models have shown strong promise for robot manipulation, yet raw language is primarily needed to specify task intent rather than to be repeatedly processed during high-frequency low-level execution. Motivated by this separation, we propose a cerebello-thalamic-inspired vision-action model (CT-VAM) for efficient task-conditioned visuomotor control. CT-VAM acts as a compact local execution policy that predicts action chunks from dual-view visual observations, proprioception, and a lightweight task condition, potentially enabling a practical cloud-edge paradigm in which high-level semantic reasoning can be handled by large models while fast closed-loop control runs on local hardware. To fuse heterogeneous inputs effectively, CT-VAM introduces TARS (Thalamic Action Routing Stream), a stream-separated conditional attention decoder that independently routes action, visual and task streams, preventing dense sensory tokens from overwhelming compact task-relevant conditions. With only 68M parameters, CT-VAM achieves LIBERO success rates competitive with substantially larger VLA models, while reducing inference latency. Together with flow-consistent inpainting for asynchronous chunk execution, CT-VAM supports high-frequency control and demonstrates robust real-world deployment on resource-constrained robotic platforms. Project page: \url{https://embodied-ai-research.github.io/ct-vam/}

\end{abstract}

\keywords{Visuomotor Policy; Stream-Separated Attention; Real-Time Action Generation} 



\section{Introduction}


Recent progress in robot learning has increasingly moved toward generalist
policies that combine visual perception, language understanding, and action
generation. Early language-based robotic systems showed that large language
models can provide useful high-level task knowledge or planning priors when
grounded by affordances or executable skills~\cite{huang2022language,
ahn2022saycan,driess2023palm}. More recent vision-language-action (VLA)
models further integrate visual observations, language instructions, and robot
actions into a unified policy, enabling broad generalization across tasks and
embodiments~\cite{brohan2022rt1,brohan2023rt2,openx2023,
black2025pi0}. These advances suggest that language is highly
valuable for task specification, semantic reasoning, and generalization.

\begin{figure}[h]
    \centering
    \includegraphics[width=1\linewidth]{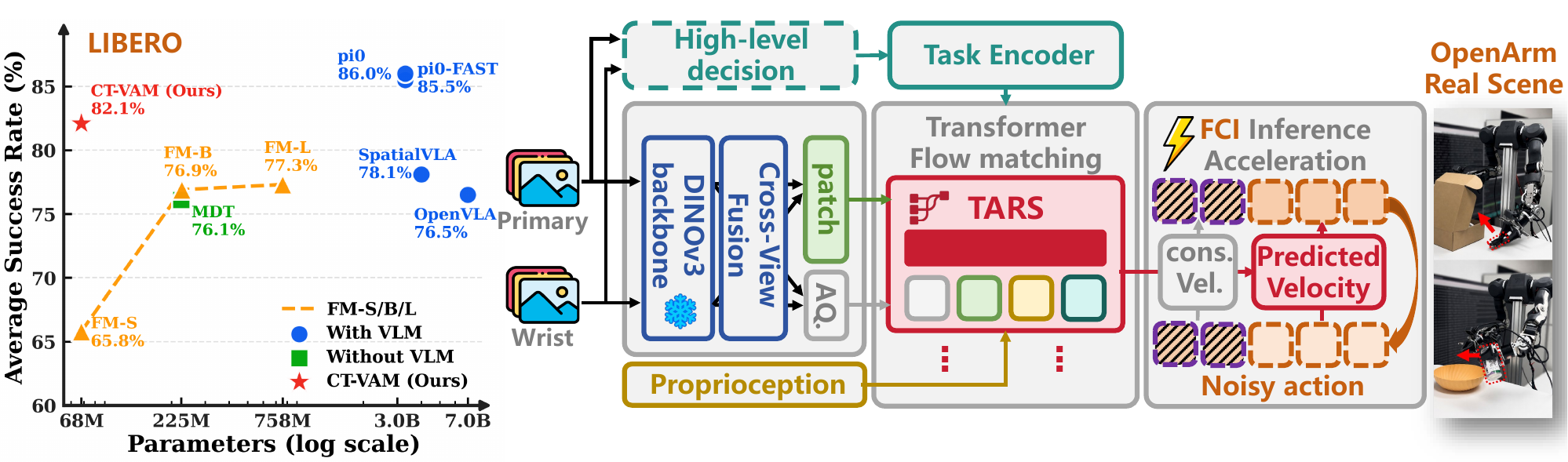}
\caption{
Overview of CT-VAM. 
Left: parameter--performance comparison on LIBERO. 
Middle: CT-VAM architecture. 
Right: real-world deployment. 
The dashed high-level decision module indicates a compatible upstream module for future integration, but it is not introduced or evaluated in this work. 
``AQ." denotes the action query. We use an explicit task encoder to isolate and assess the effectiveness of the low-level visuomotor policy.
}
    \label{fig:Pipeline}
\end{figure}

However, deploying language-conditioned policies directly in the low-level control loop introduces a practical tension. Recent VLA models connect pretrained vision-language backbones to robot action generation, enabling strong semantic generalization but also retaining substantial inference and memory costs during execution \cite{kim2025openvla}. In many manipulation tasks, the raw instruction mainly specifies the task intent, whereas fine-grained motor commands are governed by visual feedback, proprioceptive state, and the procedural dynamics learned by the policy \cite{zhao2023act,chi2025diffusion}. Repeatedly processing language-conditioned representations during receding-horizon control can therefore increase latency and hardware requirements, especially for real-world deployment on resource-constrained robotic platforms \cite{wen2024tinyvla,wang2025bitvla}.

This observation suggests an alternative organization for visuomotor control, one that echoes biological skilled manipulation: humans typically do not reinterpret a linguistic command at every motor step, but execute under sparse high-level intent through rapid sensorimotor feedback and learned motor dynamics. Motivated by this separation between intent-level guidance and fast motor execution, we decouple high-level language grounding from high-frequency visual feedback control. Language is therefore not discarded; it is used to establish or update a compact task condition when the instruction, correction, or subtask changes. Once this action-relevant condition is formed, low-level control can remain visually closed-loop without repeatedly parsing the raw instruction. Based on this principle, we propose a cerebello-thalamic-inspired vision-action model (CT-VAM) for efficient visuomotor control. Given dual-view visual observations, proprioception, and the task condition, CT-VAM predicts action chunks for closed-loop execution. Unlike billion-scale VLA policies that retain a full language-conditioned backbone during action prediction, CT-VAM targets low-level visuomotor execution and is designed for efficient closed-loop, onboard deployment.

The main contributions are summarized as follows.
\begin{itemize}
\item We formulate grounded visuomotor execution, which separates semantic
language grounding from high-frequency visual feedback control. Raw language is
converted into an intent representation, while the low-level policy still uses
visual-proprioceptive observations for closed-loop action generation, avoiding
repeated language processing during execution.

\item We introduce \textsc{TARS}, a stream-separated conditional attention module
that routes and gates action, visual, proprioceptive, and task streams,
preventing dense visual tokens from overwhelming compact but task-relevant
conditions.

\item We introduce flow-consistent inpainting for asynchronous chunk execution, enabling next-chunk inference to overlap with current execution while preserving action continuity. 
{CT-VAM achieves strong simulation and real-world performance, and supports efficient fully onboard deployment on Jetson Orin NX.}

\end{itemize}




\section{Related Work}

\textbf{Vision-Language-Action Models}. Vision-language-action (VLA) models have become a dominant direction for building generalist robot policies. Early language-conditioned robotic systems used large language models to provide task-level knowledge, planning priors, or affordance-grounded decisions \cite{huang2022language,ahn2022saycan,driess2023palm}. More recent VLA models integrate visual observations, language instructions, and robot actions into a unified policy, enabling broad generalization across tasks, scenes, and robot embodiments
\cite{brohan2022rt1,brohan2023rt2,openx2023,kim2025openvla,black2025pi0}.
This formulation is attractive because it can exploit large-scale vision-language pretraining and robot datasets, while using language as a flexible interface for task specification. Recent works further improve VLA
pretraining, spatial representation, and action tokenization \cite{ye2025lapa,qu2025spatialvla,pertsch2025fast}. However, deploying a full language-conditioned policy inside the low-level control loop remains computationally expensive. Since raw instructions often mainly specify task intent, repeatedly processing language during receding-horizon execution can increase inference latency, memory consumption, and hardware requirements, especially on resource-constrained edge devices.

\textbf{Visuomotor Policies for Action Generation}. Visuomotor imitation policies offer an efficient route to low-level robot control. Action Chunking Transformer (ACT) predicts short action sequences for fine-grained manipulation \cite{zhao2023act}, while Diffusion Policy formulates visuomotor control as conditional denoising over action trajectories \cite{chi2025diffusion}. Subsequent transformer-based diffusion and goal-conditioned policies extend this paradigm to longer-horizon and more diverse manipulation settings~\cite{reuss2024mdt,guo2026dg}, and flow matching or rectified flow provides efficient alternatives for continuous action generation~\cite{lipman2022flowmatching,liu2022rectifiedflow}. However, these policies typically fuse visual, proprioceptive, task, and action-related signals through concatenation, pooling, or shared attention, with limited explicit control over their relative contributions. Since these signals play different roles in visuomotor control, such implicit fusion may be suboptimal for action prediction.

\section{From Raw Language to Grounded Visuomotor Execution}
\label{sec:language_grounding}

This section formalizes the role of language in visuomotor control. We consider
a common execution setting where a raw instruction is first grounded into an
internal intent representation, and the subsequent low-level control is carried
out by a learned visuomotor policy. In this setting, language is needed to form
or update the intent, but it does not need to be repeatedly parsed at each
control step.

Let \(L\) denote a raw natural-language instruction, and let \(X_0\) denote the
context available at the grounding stage, such as the initial visual observation
or scene context. A grounding process maps the instruction and context into a
grounded intent representation: $G \sim q_{\phi}(\cdot \mid L, X_0)$. Here, \(G\) denotes a control-relevant intent representation, such as the one-hot
task identifier used in our implementation. More detailed interpretations of
\(G\) are provided in Appendix~\ref{app:theoretical_details}. At time \(t\), the controller observes the recent visual-proprioceptive history $O_t \triangleq \left(V_{t-t_h:t}, P_{t-t_h:t}\right)$, where \(V_t\) denotes the visual observation, \(P_t\) denotes the proprioceptive
observation, and \(t_h\) is a finite history horizon. The low-level controller is
represented by a learned visuomotor policy $A_t \sim \pi_{\theta}(\cdot \mid G, O_t)$, where \(\theta\) denotes the learned parameters of the policy. In this work, we
treat \(\theta\) as an implicit procedural memory acquired from visuomotor
training. The key question is whether the raw language instruction still provides
additional action-relevant information after the grounded intent is available.
We quantify this by the residual raw-language influence:
\begin{equation}
    \Delta_t
    \triangleq
    I(A_t; L \mid G, O_t).
    \label{eq:residual_language_influence}
\end{equation}
This quantity measures the information about the next action that remains in the
raw instruction after conditioning on both the grounded intent \(G\) and the
execution context \(O_t\).

\begin{definition}
\label{def:control_sufficient_intent}
A grounded intent representation \(G\) is \(\epsilon_t\)-control-sufficient for
the raw instruction \(L\) at time \(t\) if
\begin{equation}
    I(A_t; L \mid G, O_t) \leq \epsilon_t .
    \label{eq:epsilon_control_sufficiency}
\end{equation}
When \(\epsilon_t=0\), \(G\) is exactly control-sufficient, and the next action is
conditionally independent of the raw instruction.
\end{definition}


\begin{assumption}\label{assump:approx_action_sufficient_grounding}
The grounded intent \(G\) approximately preserves the action-relevant
information in the raw instruction \(L\). Specifically, after conditioning on
\(G\) and the execution context \(O_t\), replacing the raw instruction by the
grounded intent induces only a small change in the next-action distribution:
\begin{equation}
    D_{\mathrm{KL}}
    \left(
    p(A_t \mid l,g,o_t)
    \,\Vert\,
    p(A_t \mid g,o_t)
    \right)
    \leq \epsilon_t,
    \quad
    \forall (l,g,o_t)\in \operatorname{supp}(L,G,O_t),
    \label{eq:approx_action_sufficient_grounding}
\end{equation}
where \(\operatorname{supp}(\cdot)\) denotes the support of a distribution.
\end{assumption}

\begin{proposition}
\label{prop:epsilon_control_sufficiency}
Under Assumption~\ref{assump:approx_action_sufficient_grounding}, the grounded
intent \(G\) is \(\epsilon_t\)-control-sufficient for the raw instruction \(L\) at
time \(t\), i.e.,
\begin{equation}
    I(A_t;L \mid G,O_t) \leq \epsilon_t .
    \label{eq:prop_control_sufficiency}
\end{equation}
\end{proposition}

The proof and the scope of Assumption~\ref{assump:approx_action_sufficient_grounding}
are provided in Appendix~\ref{app:theoretical_details}. Proposition~\ref{prop:epsilon_control_sufficiency}
formalizes the role of grounding as a control-sufficient representation: once
the action-relevant intent has been captured by \(G\), the low-level policy can
condition on \(G\) and \(O_t\) without repeatedly processing the raw instruction
\(L\). If a new instruction is issued or the current subtask changes, \(G\) can
be updated accordingly; the key point is that within a subtask, \(G\) can serve
as a stable task condition for closed-loop visuomotor execution.

\section{Method}
We study task-conditioned visuomotor policy learning from synchronized dual-view
observations. At each control step, the policy receives a primary camera history
\(I^{\mathrm{pri}}_{1:T_o}\), a wrist camera history
\(I^{\mathrm{wri}}_{1:T_o}\), proprioceptive states \(s_{1:T_o}\), and a task
identifier \(y\in\{0,1\}^{N_{\mathrm{task}}}\). It predicts an action chunk 
\(a_{1:H}\) through
\begin{equation}
\pi_\theta
\left(
a_{1:H}
\,\middle|\,
I^{\mathrm{pri}}_{1:T_o},
I^{\mathrm{wri}}_{1:T_o},
s_{1:T_o},
y
\right),
\qquad
a_h\in\mathbb{R}^{d_a},
\label{eq:ct_vam_policy}
\end{equation}
where \(H\) is the action chunk length, $T_o$ is the observation history length, and \(d_a\) is the action dimension.
We denote the flattened action chunk by
\begin{equation}
A=\operatorname{vec}(a_{1:H})\in\mathbb{R}^{H d_a}.
\label{eq:flatten_action_chunk}
\end{equation}

As illustrated in Figure \ref{fig:Pipeline}, CT-VAM consists of three main
components. First, synchronized dual-view observations are encoded into
layer-wise visual memories for action generation. Second, Thalamic Action Routing Stream (\textsc{TARS}) performs stream-separated conditional attention
to regulate the interaction among action, visual, proprioceptive, and task
streams. Third, flow-consistent inpainting (FCI) enables asynchronous chunk
prediction during robot execution while preserving boundary consistency between
consecutive action chunks. Detailed visual conditioning, rectified-flow training,
and the functional cerebello-thalamic analogy are provided in
Appendix~\ref{app:visual_conditioning_details},
Appendix~\ref{app:rectified_flow_details}, and
Appendix~\ref{app:functional_analogy}, respectively.

\begin{figure}[h]
    \centering
    \includegraphics[width=0.85\linewidth]{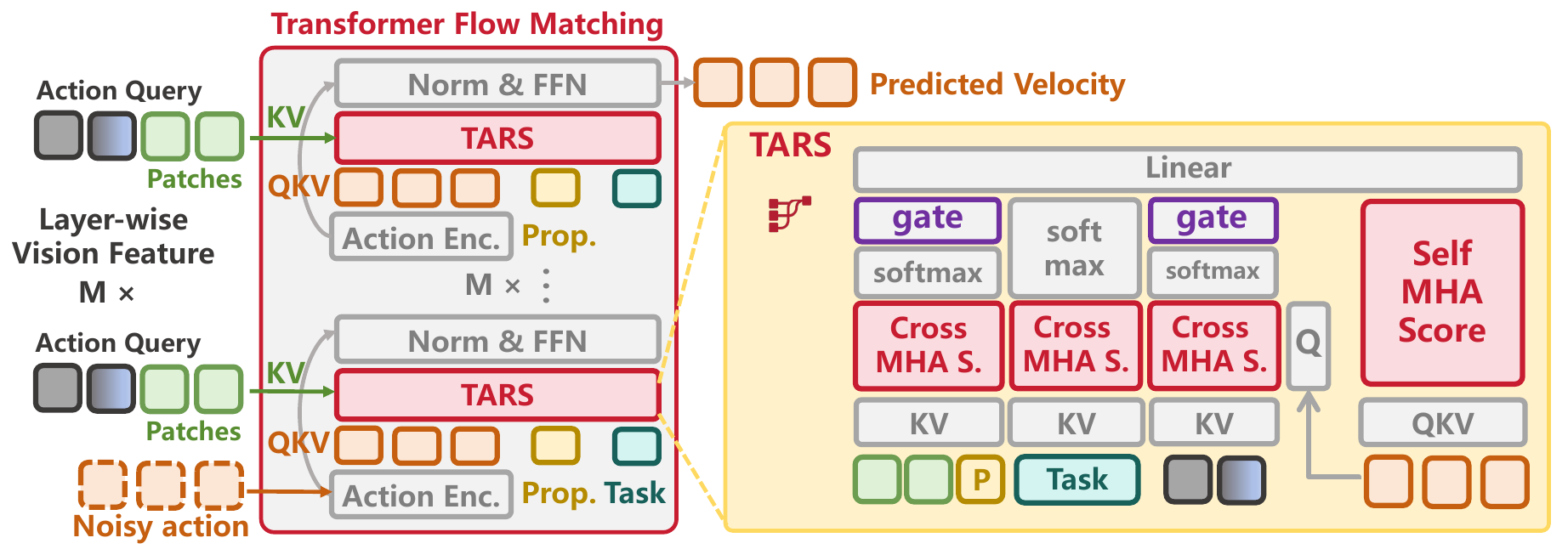}
    \caption{Overview of the proposed TARS.
    TARS updates action tokens by attending to four separated streams: the
    current action tokens, action query, dense visual-proprioceptive
    memories, and the task token. Each stream is normalized independently before
    gated aggregation.}
    \label{fig:tars_framework}
\end{figure}

\subsection{Thalamic Action Routing Stream}
\label{subsec:tars}

The action decoder is a transformer-based rectified-flow model. At flow time
\(\tau\in[0,1]\), TARS operates on a noised action state \(A_\tau\), which is
embedded into the initial action tokens before the first decoder layer
\begin{equation}
X_{0}
=
\phi_a(\operatorname{reshape}(A_\tau))
+
E_{\mathrm{act}}
\in
\mathbb{R}^{H\times d},
\label{eq:action_token_embedding}
\end{equation}
where \(\phi_a\) is the action encoder, \(E_{\mathrm{act}}\) is a learnable
action positional embedding, \(H\) is the action chunk length, and \(d\) is the
decoder hidden dimension.

TARS also maintains a set of learnable action queries $E_{\mathrm{aq}}\in\mathbb{R}^{N_q\times d}$. These action queries are initialized from the register tokens of the pretrained visual backbone \cite{Simeoni2025DINOv3} and are then optimized jointly with the policy.
They serve as decoder-side action refinement slots for organizing action-chunk
generation. The detailed initialization and expansion rule is provided in
Appendix~\ref{app:action-query-initialization}.

For each layer \(\ell\), the dual-view encoder provides a dense spatial
memory \(D_\ell\), which is constructed from fused patch tokens. The
proprioceptive history is encoded as
\begin{equation}
S
=
\phi_s(s_{1:T_o})
\in
\mathbb{R}^{T_o\times d},
\label{eq:proprio_embedding}
\end{equation}
where \(\phi_s(\cdot)\) is the proprioceptive encoder. The task identifier is
embedded into a task token:
\begin{equation}
r
=
yW_{\mathrm{task}},
\qquad
W_{\mathrm{task}}
\in
\mathbb{R}^{N_{\mathrm{task}}\times d}.
\label{eq:task_embedding_main}
\end{equation}

At decoder layer \(\ell\), TARS conditions the action tokens on four separated
streams:
\begin{equation}
\begin{aligned}
\mathcal{M}_{\ell}^{\mathrm{self}}  = X_\ell, \quad
\mathcal{M}_{\ell}^{\mathrm{aq}}    = E_{\mathrm{aq}}, \quad
\mathcal{M}_{\ell}^{\mathrm{dense}} = [D_\ell; S], \quad
\mathcal{M}_{\ell}^{\mathrm{task}}  = \{r\}.
\end{aligned}
\label{eq:tars_streams}
\end{equation}
The self stream represents the current action tokens, the action-query stream
provides learnable register-initialized action refinement slots, the dense
stream provides patch-level visual evidence and proprioceptive information, and
the task stream provides the compact task condition.

Queries are produced from the action tokens, while each memory stream has its
own key-value projection:
\begin{equation}
\begin{aligned}
Q_{\ell}^{a}
&=
\operatorname{LN}(X_\ell)W_q, \\
K_{\ell}^{b}
&=
\operatorname{LN}(\mathcal{M}_{\ell}^{b})W_k^{b}, \\
V_{\ell}^{b}
&=
\operatorname{LN}(\mathcal{M}_{\ell}^{b})W_v^{b}.
\end{aligned}
\label{eq:tars_qkv}
\end{equation}
Here \(b\in\mathcal{B}\), with $\mathcal{B}
=
\{\mathrm{self},\mathrm{aq},\mathrm{dense},\mathrm{task}\}$. For each stream \(b\), we compute the stream-specific attention output:
\begin{equation}
\mathcal{A}_{b}
=
\operatorname{Softmax}
\left(
\frac{
Q_{\ell}^{a}(K_{\ell}^{b})^{\top}
}{
\sqrt{d_h}
}
\right)
V_{\ell}^{b},
\label{eq:tars_stream_attention}
\end{equation}
where \(d_h\) is the attention head dimension.

The stream outputs are aggregated by explicit stream gates:
\begin{equation}
\hat{X}_{\ell}
=
X_{\ell}
+
\left(
\sum_{b\in\mathcal{B}}
\gamma_b\mathcal{A}_b
\right)
W_o,
\label{eq:tars_gated_aggregation}
\end{equation}
where \(\gamma_{\mathrm{self}}=\gamma_{\mathrm{task}}=1\), \(W_o\) is the
output projection matrix, and \(\gamma_{\mathrm{aq}}\) and
\(\gamma_{\mathrm{dense}}\) are learned gates.

The layer output is then updated by a time-conditioned feed-forward block:
\begin{equation}
X_{\ell+1}
=
\hat{X}_{\ell}
+
\operatorname{FFN}
\left(
\operatorname{LN}(\hat{X}_{\ell})+e(\tau)
\right),
\label{eq:tars_ffn}
\end{equation}
where \(X_\ell\) denotes the action tokens at decoder layer \(\ell\).

Unlike shared-softmax attention in \cite{wang2026vlaadapter}, TARS normalizes each stream independently
before aggregation. This prevents dense visual-proprioceptive memories from
dominating compact but task-relevant streams simply due to their larger token
count. A detailed comparison between shared-softmax attention and
stream-separated normalization is provided in
Appendix~\ref{app:tars_stream_normalization}.

\subsection{Efficient On-Device Inference and Real-Time Execution}
\label{subsec:efficient_inference}

By removing raw language from the low-level execution loop and conditioning the
policy only on visual observations, proprioception, and a compact task
representation, CT-VAM keeps the model lightweight with only 68M parameters when using the DINOv3-S+ visual backbone \cite{Simeoni2025DINOv3}).
This compact design significantly reduces the per-chunk inference latency.
Detailed latency comparisons are provided in
Appendix~\ref{app:real_time_equivalent_analysis}. Nevertheless, this latency
still cannot be ignored in real-time robot execution. Since executing an action
chunk already consumes physical time, synchronous chunk generation adds
inference latency on top of the action execution time, thereby slowing down
closed-loop control.

Therefore, the remaining inference time should be overlapped with action
execution rather than being paid sequentially. Since CT-VAM has a low memory
footprint, it supports an asynchronous inference pipeline that predicts the next
action chunk while the current chunk is still being executed. The remaining
challenge is that independently generated chunks may be inconsistent at their
boundary. We therefore introduce \emph{flow-consistent inpainting} (FCI), which
imposes rectified-flow-consistent overlap constraints to ensure smooth
transitions between consecutive chunks.

To quantify the benefit of asynchronous execution, we use a real-time equivalent
execution metric that separates physical action execution from inference latency
exposed on the control path. For one episode executed at a target control period
\(\Delta t\), we define
\begin{equation}
T
=
N_{\mathrm{actions}}\Delta t
+
S_{\mathrm{method}},
\label{eq:tmethod-v2}
\end{equation}
where $T$ is the real-time equivalent episode duration, \(N_{\mathrm{actions}}\) is the total number of low-level actions executed
in the episode, and \(S_{\mathrm{method}}\) is the \emph{exposed stall}, i.e.,
the portion of inference latency that cannot be hidden by action execution and
must still be paid sequentially. This metric excludes simulator-specific
overheads such as rendering or physics stepping.

For the proposed FCI scheme, the first action chunk is generated synchronously.
The robot then starts executing this chunk, while the next inference is launched
in parallel, as illustrated in Figure~\ref{fig:flow_consistent_inpainting}. When
the current chunk has \(K_{\mathrm{ov}}\) actions remaining, these remaining
actions define an overlap segment
\begin{equation}
Y_{\mathrm{ov}}\in\mathbb{R}^{K_{\mathrm{ov}}\times d_a}.
\label{eq:overlap_segment}
\end{equation}
The new chunk should match this segment at its beginning so that switching from
the current chunk to the new chunk does not introduce a discontinuity in the
executed action stream.

\begin{figure}[htb]
    \centering
    \includegraphics[width=0.85\linewidth]{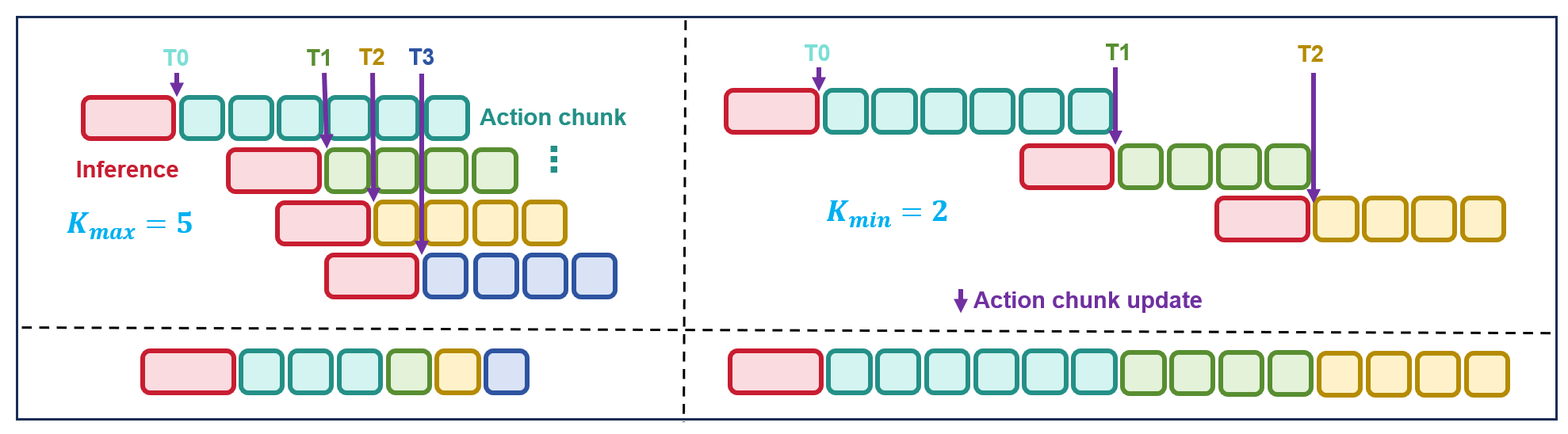}
    \caption{{Illustration of the proposed flow-consistent inpainting scheme under the maximum (left) and minimum (right) \(K_{\mathrm{ov}}\) settings.
    The next action chunk is predicted before the current chunk is exhausted,
    and its overlap region is constrained to follow the rectified-flow trajectory
    toward the remaining actions of the current chunk. This ensures boundary
    continuity while keeping the decoder input consistent with the training
    distribution.}}
    \label{fig:flow_consistent_inpainting}
\end{figure}

A direct hard clamp of the overlap actions would satisfy boundary continuity but
would break the rectified-flow noise-to-data trajectory used during training.
Instead, FCI constrains the overlap region to follow the same interpolation form
as rectified-flow training. For each new chunk, we sample a fixed overlap noise:
\begin{equation}
\epsilon_{\mathrm{ov}}
\sim
\mathcal{U}[-1,1]^{K_{\mathrm{ov}}\times d_a}.
\label{eq:overlap_noise}
\end{equation}
At flow step \(i\), with
\begin{equation}
\tau_i
=
\frac{i}{N_{\mathrm{FE}}},
\label{eq:flow_time_fci}
\end{equation}
where ${N_{\mathrm{FE}}}$ is the number of rectified-flow integration steps. The first \(K_{\mathrm{ov}}\) actions of the new chunk are constrained as
\begin{equation}
A^{(i)}_{1:K_{\mathrm{ov}}}
=
(1-\tau_i)\epsilon_{\mathrm{ov}}
+
\tau_i Y_{\mathrm{ov}} ,
\label{eq:flow_consistent_overlap}
\end{equation}
where $A^{(i)}$ denotes the new chunk state at flow step $i$. The remaining free actions are updated by the standard rectified-flow
integration, while the overlap constraint in
Eq.~\eqref{eq:flow_consistent_overlap} is re-applied after each flow step. The detailed update rule and comparison with hard inpainting, together with its advantages, are provided in
Appendix~\ref{app:real_time_equivalent_analysis}.

\section{Experiments}

\subsection{Experiments on LIBERO Benchmark}

We evaluate CT-VAM on the LIBERO benchmark, including LIBERO-Spatial, LIBERO-Object, LIBERO-Goal, and LIBERO-Long, which cover spatial reasoning, object-centric manipulation, goal-conditioned execution, and long-horizon task completion. We follow the standard LIBERO protocol and report task success rate, averaged over 50 rollouts per task. The policy takes two visual observations, proprioceptive states, and a one-hot task identifier as input, without invoking a large VLM during low-level execution. Unless otherwise stated, CT-VAM uses TARS with 32 learnable action queries, as detailed in Appendix~\ref{app:action-query-initialization}.

\begin{table*}[htb]
\centering
\caption{Performance comparison on the LIBERO benchmark across VLM-based and non-VLM visuomotor policies.}
\label{tab:libero}
\setlength{\tabcolsep}{5pt}
\renewcommand{\arraystretch}{1.08}

\begin{tabular}{lc|cccc|c}
\toprule
\multicolumn{1}{c}{\textbf{LIBERO}} 
& \textbf{Params} 
& \textbf{Spatial} 
& \textbf{Object} 
& \textbf{Goal} 
& \textbf{Long} 
& \textbf{Avg.} \\
\midrule

\multicolumn{7}{c}{\textit{Policy \textbf{With} VLM}} \\

OpenVLA~\cite{kim2025openvla} \venue{CoRL}
& 7000M & 84.7 & 88.4 & 79.2 & 53.7 & 76.5 \\

SpatialVLA~\cite{qu2025spatialvla} \venue{RSS}
& 4000M & 88.2 & 89.9 & 78.6 & 55.5 & 78.1 \\

$\pi_0$-FAST~\cite{pertsch2025fast} \venue{RSS}
& 3300M & \underline{96.4} & \underline{96.8} & 88.6 & 60.2 & 85.5 \\

$\pi_0$~\cite{black2025pi0} \venue{RSS}
& 3300M & 90.0 & 86.0 & \underline{95.0} & \underline{73.0} & 86.0 \\

\midrule

\multicolumn{7}{c}{\textit{Policy \textbf{Without} VLM}} \\

Diffusion Policy~\cite{chi2025diffusion} \venue{IJRR}
& -- & 78.3 & 92.5 & 68.3 & 50.5 & 72.4 \\

MDT~\cite{reuss2024mdt} \venue{RSS}
& $\sim$225M & 78.5 & 87.5 & 73.5 & 64.8 & 76.1 \\

\rowcolor{tablerowgray}
\textbf{CT-VAM (Ours)}
& \textbf{68M} & \textbf{89.0} & \textbf{94.6} & \textbf{78.4} & \textbf{66.2} & \textbf{82.1} \\

\bottomrule
\end{tabular}
\end{table*}

Table~\ref{tab:libero} compares CT-VAM with representative VLM-based and non-VLM visuomotor policies. With only 68M parameters, CT-VAM achieves competitive overall performance and outperforms existing non-VLM baselines on average, while remaining much smaller than billion-scale VLM policies. These results support our central claim that, once task intent is grounded, strong closed-loop visuomotor execution can be achieved by a compact vision-action policy without repeatedly using a large language-conditioned controller.

We further conduct two diagnostic studies on LIBERO, with detailed results reported in Appendix~\ref{app:aq_ablation_libero} and \ref{app:scaling_libero}. First, we ablate the number and learnability of TARS action queries, showing that 32 learnable queries provide the best balance between action representation capacity and optimization stability. Second, we study scaling behavior across flow-matching policies with different overall model sizes. The small, base, and large variants achieve average success rates of 65.8\%, 76.9\%, and 77.3\%, respectively, while CT-VAM reaches 82.1\%, outperforming even the largest variant. This indicates that CT-VAM's gain does not come from model size alone, but from the proposed stream-routing and action-query design.

\subsection{Real-World Experiments}
\label{app:real-world exper}

\begin{figure}[h]
    \centering
\includegraphics[width=1\linewidth]{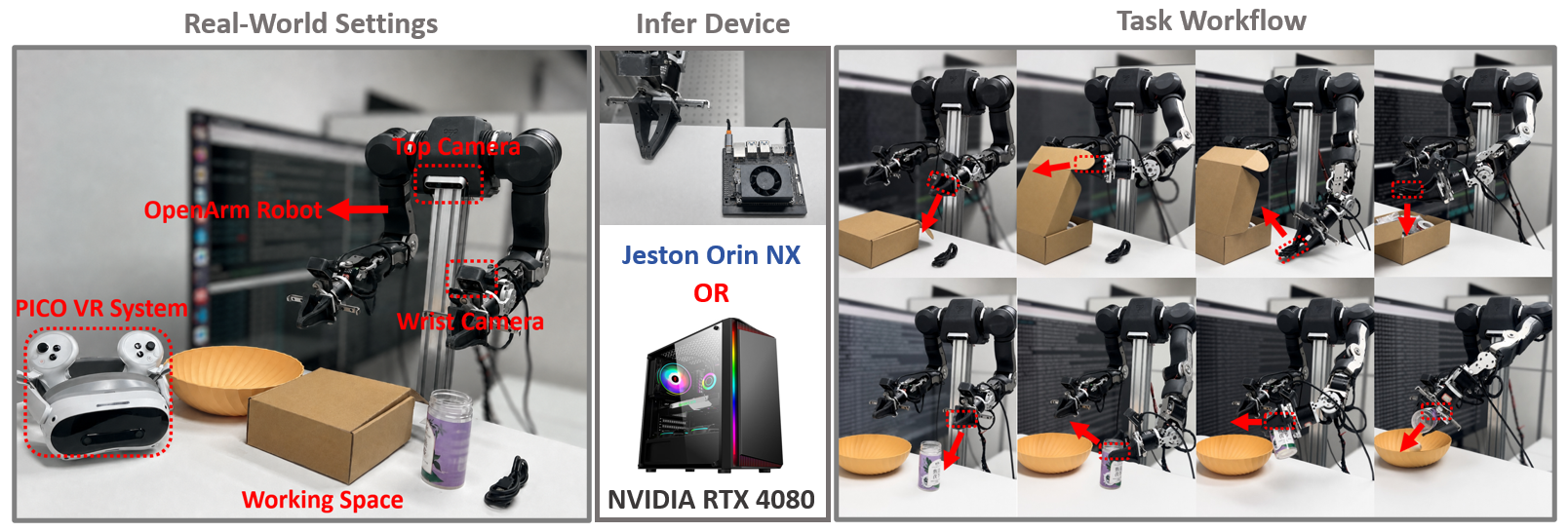}
    \caption{Real-world experimental setup and task workflow. Left: the OpenArm platform. Middle: inference devices used for deployment, including Jetson Orin NX and NVIDIA RTX 4080. Right: representative execution sequences of the real-world manipulation tasks, including box opening and object placement, and bottle pouring into a target bowl.}

\label{fig:real_world_setup}
\end{figure}

\textbf{Real-World Deployment}. 
We conduct real-world experiments on the OpenArm platform to evaluate the closed-loop manipulation performance of CT-VAM. The hardware setup, camera configuration, teleoperation device, and computing platforms are detailed in Appendix~\ref{app:real_world_exper}, and the Jetson Orin NX deployment with TensorRT optimization is described in Appendix~\ref{app:jetson_trt_deployment}. During deployment, CT-VAM is integrated
into the OpenArm control framework as a left-arm joint-space controller. The
policy takes third-person RGB observations, wrist RGB observations, and robot
proprioceptive states as inputs, and predicts short-horizon joint commands
together with gripper actions in a closed-loop receding-horizon manner. For data collection, we collect 30 demonstration episodes for each task using
teleoperation and train all models with a batch size of 32. To ensure a fair
comparison, all methods use the same demonstration data, observation-action
space, and evaluation protocol. During evaluation, object and target
configurations are varied within the reachable workspace of the active arm. We
conduct 20 trials for each task and report the task success rate and average
task execution time as the primary evaluation metrics. A trial is considered
successful only if the robot completes the task without human intervention. Following Sec.~\ref{subsec:efficient_inference}, CT-VAM executes actions in a
closed-loop receding-horizon manner. When FCI is enabled,
the next action chunk is generated asynchronously while the current chunk is
being executed, and the overlapping actions are constrained to ensure smooth
transitions between consecutive chunks. The predicted commands are then sent to
the low-level OpenArm controller for execution.

\textbf{Real-World Results}. We evaluate CT-VAM on two real-world tabletop manipulation tasks. Ball Pouring is used for quantitative comparison with baseline methods and ablated variants, while Box Opening and Placement is used to examine long-horizon task-conditioned execution under subtask switching. Table~\ref{tab:real_world_ball_pouring_platform_methods} reports the Ball Pouring results across two deployment platforms. Each method-platform setting is evaluated over 20 trials, and we report task success rate, execution time, and inference latency. Following the reference configurations, $\pi_0$ is executed with 50 inference steps, while CT-VAM and Diffusion Policy use 8 inference steps. On the RTX 4080 platform, CT-VAM achieves a success rate comparable to $\pi_0$ and clearly higher than Diffusion Policy. The comparison between CT-VAM variants shows that FCI maintains comparable task success while reducing execution time by overlapping next-chunk inference with
current action execution. Under a 20 Hz low-level control frequency, FCI further
allows the real-time-equivalent deployment frequency to approach the target
20 Hz control rate by hiding most inference latency inside physical action
execution; the detailed calculation is provided in
Appendix~\ref{app:real_time_equivalent_frequency}. The Jetson Orin NX results
further demonstrate that CT-VAM can be deployed on a resource-constrained edge
platform, with the TensorRT-optimized inference backend detailed in
Appendix~\ref{app:jetson_trt_deployment}.

\begin{table*}[t]
    \centering
    \caption{Real-world performance on the Ball Pouring task across different methods and deployment platforms. Each method-platform setting is evaluated over 20 trials. We report the task success rate, execution time, and average inference time.}
    \label{tab:real_world_ball_pouring_platform_methods}
    \setlength{\tabcolsep}{5pt}
    \renewcommand{\arraystretch}{1.08}
    \resizebox{\textwidth}{!}{
    \begin{tabular}{lccc ccc}
        \toprule
        \multirow{2}{*}{Method}
        & \multicolumn{3}{c}{RTX 4080}
        & \multicolumn{3}{c}{Jetson Orin NX} \\
        \cmidrule(lr){2-4} \cmidrule(lr){5-7}
        & Success Rate (\%) $\uparrow$
        & Exec. Time (s) $\downarrow$
        & Infer. Time (ms) $\downarrow$
        & Success Rate (\%) $\uparrow$
        & Exec. Time (s) $\downarrow$
        & Infer. Time (ms) $\downarrow$ \\
        \midrule
        Diffusion Policy
        & 70.0 & 27.21 & 303.27
        & N.T. & N.T. & N.T. \\

        $\pi_0$
        & 95.0 & 7.82 & 117.24
        & N.D. & N.D. & N.D. \\

        CT-VAM w/o FCI
        & \textbf{100} & 8.33 & \textbf{56.32}
        & 85 & 10.24 & 256.24 \\

        CT-VAM w/ FCI
        & 95 & \textbf{6.41} & 56.84
        & \textbf{90} & \textbf{7.23} & \textbf{200.60} \\
        \bottomrule
    \end{tabular}
    }
\footnotesize{\emph{Note:} N.T. indicates that the method was not tested on Jetson Orin NX. N.D. indicates that the method could not be deployed due to memory constraints.}
\end{table*}

\label{sec:result}

\section{Limitations}
This work has several limitations. First, CT-VAM currently uses one-hot task
tokens as compact intent representations and does not yet include a full
language grounding module. This design isolates the evaluation of the low-level
visuomotor policy, but does not address open-ended instruction grounding. Future work will integrate a dedicated grounding module that maps raw language and scene context into task-relevant intent representations, following recent progress in language-conditioned manipulation and language-to-spatial grounding for robot control~\cite{shridhar2022cliport,shridhar2023peract,huang2023voxposer}. Second, our real-world
evaluation is limited to tabletop tasks and uses manual task switching in the
long-horizon experiment. Extending CT-VAM to automatic subtask recognition,
broader task distributions, and more diverse robot embodiments remains an
important direction.


\section{Conclusion}
\label{sec:conclusion}

We presented CT-VAM, a compact vision-action model for efficient visuomotor control. By separating language grounding from low-level visual feedback control, CT-VAM avoids repeated language-conditioned inference during execution. The proposed TARS module improves heterogeneous stream fusion, while FCI enables asynchronous chunk execution with smooth action transitions. Experiments on LIBERO and real-world manipulation tasks show that CT-VAM achieves competitive performance with low inference cost, supporting practical deployment on resource-constrained robotic platforms.

\bibliography{example}  

\clearpage
\appendix
\section{Theoretical Details}
\label{app:theoretical_details}

\subsection{Interpretation of Grounded Intent}
\label{app:grounded_intent_interpretation}

The grounded intent \(G\) is not assumed to be a language string. It may be an
explicit task identifier, a latent instruction feature, a goal-like
representation, or a structured state that encodes the task-relevant information
required for control. In our implementation, the one-hot task identifier follows
this setting and serves as a compact intent representation for the considered
task family.

Importantly, \(G\) does not need to be semantically equivalent to the original
language instruction \(L\). Different surface forms of language may correspond
to the same grounded intent if they imply the same robot behavior. Conversely,
language-specific details such as politeness, surface wording, syntactic form,
and other redundant expressions may be discarded as long as they do not affect
the required control behavior.

The residual influence \(\Delta_t=I(A_t;L\mid G,O_t)\) is different from
\(I(A_t;L\mid O_t)\). The latter measures whether language provides additional
information about the action when only the visual-proprioceptive history is
given. In general, \(I(A_t;L\mid O_t)\) can remain large, because the same
observation history may be compatible with multiple task goals. In contrast,
\(\Delta_t\) asks whether the raw instruction still matters after the grounded
intent \(G\) has already been formed.

\subsection{Rationale and Scope of Assumption~\ref{assump:approx_action_sufficient_grounding}}
\label{app:assumption_rationale}

Assumption~\ref{assump:approx_action_sufficient_grounding} is based on the
classical idea of sufficient statistics: a representation can replace the
original input for a downstream decision if it retains the information in the
input that is relevant to that decision~\citep{fisher1922mathematical,
cover2006elements}. In our setting, the original input is the raw instruction
\(L\), the learned representation is the grounded intent \(G\), and the
downstream decision is the next action \(A_t\) under the execution context
\(O_t\). Therefore, if the grounding process is accurate enough, \(G\) should
preserve the action-relevant content of \(L\), while discarding language details
that do not affect control. This is also consistent with the information
bottleneck view, where a useful representation compresses the input while
preserving information relevant to the target variable~\citep{tishby1999information}.

This assumption is not meant to hold for an arbitrary encoded feature. It only
holds when the grounding module has learned an intent representation that is
sufficient for the considered task family. In particular, the tuple
\((l,g,o_t)\) in Eq.~\eqref{eq:approx_action_sufficient_grounding} is not an
arbitrary combination of language and intent; \(g\) is the grounded
representation induced by \(l\) and the initial context through the grounding
process. If the instruction contains object bindings, constraints, subgoals, or
execution preferences that are not encoded in \(G\), then the KL divergence in
Eq.~\eqref{eq:approx_action_sufficient_grounding} may be large, and the raw
instruction may still provide additional information for action generation.
Thus, Assumption~\ref{assump:approx_action_sufficient_grounding} should be read
as a condition on the quality of grounding rather than as a universal property
of all language encoders.

\subsection{Proof of Proposition~\ref{prop:epsilon_control_sufficiency}}
\label{app:proof_control_sufficiency}

\begin{proof}
By the KL-divergence form of conditional mutual information,
\begin{align}
    I(A_t;L\mid G,O_t)
    =
    \mathbb{E}_{p(l,g,o_t)}
    \left[
    D_{\mathrm{KL}}
    \left(
    p(A_t\mid l,g,o_t)
    \,\Vert\,
    p(A_t\mid g,o_t)
    \right)
    \right].
    \label{eq:cmi_kl_form_grounding}
\end{align}
Assumption~\ref{assump:approx_action_sufficient_grounding} bounds each KL term
inside the expectation by \(\epsilon_t\). Therefore,
\begin{equation}
    I(A_t;L\mid G,O_t)
    \leq
    \mathbb{E}_{p(l,g,o_t)}[\epsilon_t]
    =
    \epsilon_t .
\end{equation}
This is exactly the condition in
Definition~\ref{def:control_sufficient_intent}. Hence, \(G\) is
\(\epsilon_t\)-control-sufficient for \(L\) at time \(t\).
\end{proof}

\section{Details of Proposed Method}
\label{app:method_details}

\subsection{Functional Analogy}
\label{app:functional_analogy}

The term \emph{cerebello-thalamic-inspired} denotes a functional analogy rather
than a literal biological model. CT-VAM is motivated by a modular separation
between high-level intent processing and low-level visuomotor execution. In this
view, VAM serves as a compact cerebellum-like visuomotor module that stores
learned procedural control knowledge and generates low-level action chunks from
visual-proprioceptive observations and task conditions. A high-level VLM can be
attached above it as a brain-like module for language understanding, reasoning,
and subtask-level intent updates.

This separation keeps costly language processing outside the high-frequency
control loop while preserving closed-loop visuomotor control. Within VAM, TARS
provides the thalamic part of the analogy by routing and gating heterogeneous
action, visual, proprioceptive, and task streams through stream-separated
attention, thereby regulating how different information pathways contribute to
the action update.

\subsection{Dual-View Visual Conditioning Details}
\label{app:visual_conditioning_details}

We encode both camera streams using a frozen ViT backbone~\cite{Simeoni2025DINOv3}.
For view \(v\in\{\mathrm{pri},\mathrm{wri}\}\), observation step \(t\), and
visual layer \(\ell\), the backbone outputs
\begin{equation}
Z^v_{\ell,t}
=
[c^v_{\ell,t}; R^v_{\ell,t}; P^v_{\ell,t}]
\in
\mathbb{R}^{(1+R+N)\times d_b},
\label{eq:appendix_vit_tokens}
\end{equation}
where \(c^v_{\ell,t}\), \(R^v_{\ell,t}\), and \(P^v_{\ell,t}\) denote the class,
register, and patch tokens, respectively. Here \(R\) is the number of register
tokens, \(N\) is the number of patch tokens, and \(d_b\) is the backbone feature
dimension. All tokens are projected to the policy width \(d\):
\begin{equation}
\hat Z^v_{\ell,t}=Z^v_{\ell,t}W_{\mathrm{vis}},
\qquad
W_{\mathrm{vis}}\in\mathbb{R}^{d_b\times d}.
\label{eq:appendix_visual_projection}
\end{equation}

To combine the two views, we use wrist-query cross-attention. Starting from
\begin{equation}
F^{(0)}_{\ell,t}=\hat Z^{\mathrm{wri}}_{\ell,t},
\label{eq:appendix_fusion_init}
\end{equation}
each fusion block updates
\begin{equation}
\begin{aligned}
\bar F^{(m+1)}_{\ell,t}
&=
F^{(m)}_{\ell,t}
+
\operatorname{MHA}
\left(
\operatorname{LN}(F^{(m)}_{\ell,t}),
\operatorname{LN}(\hat Z^{\mathrm{pri}}_{\ell,t}),
\operatorname{LN}(\hat Z^{\mathrm{pri}}_{\ell,t})
\right),\\
F^{(m+1)}_{\ell,t}
&=
\bar F^{(m+1)}_{\ell,t}
+
\operatorname{MLP}
\left(
\operatorname{LN}(\bar F^{(m+1)}_{\ell,t})
\right).
\end{aligned}
\label{eq:appendix_wrist_query_fusion}
\end{equation}
After \(M_f\) fusion blocks, the fused token bank is
\begin{equation}
\tilde Z_{\ell,t}
=
F^{(M_f)}_{\ell,t}
=
[\tilde c_{\ell,t};\tilde R_{\ell,t};\tilde P_{\ell,t}]
\in
\mathbb{R}^{(1+R+N)\times d}.
\label{eq:appendix_fused_tokens}
\end{equation}

From the fused token bank, we construct the dense spatial memory using the
patch tokens:
\begin{equation}
D_\ell
=
\operatorname{Concat}_{t=1}^{T_o}
\tilde P_{\ell,t}
\in
\mathbb{R}^{T_oN\times d}.
\label{eq:appendix_dense_spatial_memory}
\end{equation}
Here \(D_\ell\) denotes the patch-token-based dense spatial memory at visual
layer \(\ell\). The class/register tokens are not used as an additional
observation-dependent memory stream in TARS; instead, the decoder-side action
queries are initialized from the pretrained class/register global tokens and
then optimized jointly with the policy, as described in
Appendix~\ref{app:action-query-initialization}.

The task identifier \(y\in\{0,1\}^{N_{\mathrm{task}}}\) is embedded as
\begin{equation}
r
=
yW_{\mathrm{task}},
\qquad
W_{\mathrm{task}}
\in
\mathbb{R}^{N_{\mathrm{task}}\times d}.
\label{eq:appendix_task_embedding}
\end{equation}
Keeping \(r\) separate from dense visual tokens allows the decoder to access
task information without mixing it into the patch-level spatial memory.


\subsection{Rectified-Flow Training and Inference Details}
\label{app:rectified_flow_details}

We train the decoder to predict the velocity field between a base action sample
and the expert action chunk. For an expert chunk
\(a_{1:H}\), we use its flattened representation
\begin{equation}
A=\operatorname{vec}(a_{1:H})\in\mathbb{R}^{H d_a}.
\label{eq:appendix_flattened_action}
\end{equation}
During training, we sample
\begin{equation}
\tau\sim\operatorname{Beta}(2,2),
\qquad
\tau\leftarrow\operatorname{clip}(\tau,\tau_{\min},\tau_{\max}),
\qquad
\epsilon\sim\mathcal U[-1,1]^{H d_a},
\label{eq:appendix_rf_sampling}
\end{equation}
and construct the linear interpolation
\begin{equation}
A_\tau=(1-\tau)\epsilon+\tau A.
\label{eq:appendix_rf_interpolation}
\end{equation}
The target velocity is therefore
\begin{equation}
v^\star=A-\epsilon.
\label{eq:appendix_rf_velocity_target}
\end{equation}

The decoder predicts a conditional velocity field:
\begin{equation}
v_\theta(A_\tau,\tau,\mathcal C),
\qquad
\mathcal C
=
\left\{
\mathcal M_\ell^{\mathrm{aq}},
\mathcal M_\ell^{\mathrm{dense}},
\mathcal M_\ell^{\mathrm{task}}
\right\}_{\ell=1}^{L},
\label{eq:appendix_velocity_field}
\end{equation}
where \(\mathcal C\) denotes the layer-wise conditioning streams used by TARS.
Specifically, \(\mathcal M_\ell^{\mathrm{aq}}\) is the learnable action-query
stream, \(\mathcal M_\ell^{\mathrm{dense}}=[D_\ell;S]\) is the dense
visual-proprioceptive stream, and \(\mathcal M_\ell^{\mathrm{task}}=\{r\}\) is
the task stream. The self stream is not included in \(\mathcal C\), because it
is produced from the noised action state \(A_\tau\), which is already an
explicit input to the velocity field.
After the final decoder layer, we regress the velocity from the
action tokens:
\begin{equation}
v_\theta(A_\tau,\tau,\mathcal C)
=
W_{\mathrm{out}}
\operatorname{vec}
\left(
\operatorname{LN}(X_L)
\right)
+
b_{\mathrm{out}} .
\label{eq:appendix_velocity_regression}
\end{equation}
where \(W_{\mathrm{out}}\) and \(b_{\mathrm{out}}\) are the output projection matrix and bias. The training objective is the masked element-wise mean squared error:
\begin{equation}
\mathcal L_{\mathrm{RF}}
=
\frac{1}{\sum_j m_j}
\sum_{j=1}^{H d_a}
m_j
\left(
v_{\theta,j}(A_\tau,\tau,\mathcal C)
-
v^\star_j
\right)^2 .
\label{eq:appendix_rf_loss}
\end{equation}

At inference time, we initialize \(A^{(0)}\sim\mathcal U[-1,1]^{H d_a}\) and
integrate
\begin{equation}
A^{(n+1)}
=
A^{(n)}
+
\Delta\tau\,
v_\theta
\left(
A^{(n)},\tau_n,\mathcal C
\right),
\qquad
\tau_n=\frac{n}{N_{\mathrm{FE}}},
\qquad
\Delta\tau=\frac{1}{N_{\mathrm{FE}}}.
\label{eq:appendix_rf_integration}
\end{equation}
The final vector is reshaped into the predicted action chunk \(a_{1:H}\).


\subsection{Action-Query Initialization and Expansion}
\label{app:action-query-initialization}

The action-query stream in CT-VAM is initialized from the register tokens of the
pretrained DINOv3 visual backbone~\cite{Simeoni2025DINOv3}, rather than from
purely random latent slots. Specifically, for the ViT-S+ backbone, DINOv3
provides four register tokens. Let \(N_0=4\) denote the number of register
tokens, and let \(d_b\) denote the backbone feature dimension. The pretrained
register-token embeddings are written as
\begin{equation}
E^{\mathrm{glob}}
=
\begin{bmatrix}
e^{\mathrm{reg}}_1 \\
\vdots \\
e^{\mathrm{reg}}_{N_0}
\end{bmatrix}
\in
\mathbb{R}^{N_0 \times d_b}.
\label{eq:global_register_tokens}
\end{equation}

These tokens are first projected to the policy hidden dimension:
\begin{equation}
Q^{\mathrm{proto}}
=
E^{\mathrm{glob}} W_{\mathrm{vis}}
\in
\mathbb{R}^{N_0 \times d},
\label{eq:action_query_prototypes}
\end{equation}
where \(W_{\mathrm{vis}}\in\mathbb{R}^{d_b\times d}\) is the visual projection
matrix. The projected register tokens are used as action-query prototypes.

When the number of action queries satisfies \(N_q=N_0\), we directly use these
projected register tokens as the initial action queries. When a larger number of
action queries is required, such as \(N_q=32\), we expand the query set by
duplicating the pretrained prototypes and adding small random perturbations:
\begin{equation}
Q^{\mathrm{act}}_j
=
Q^{\mathrm{proto}}_{1 + ((j-1) \bmod N_0)}
+
\sigma_q \xi_j,
\qquad
j=1,\ldots,N_q,
\label{eq:action_query_expansion}
\end{equation}
where \(\xi_j\in\mathbb{R}^{d}\) is a zero-mean random perturbation and
\(\sigma_q\) controls the perturbation scale. This gives
\begin{equation}
Q^{\mathrm{act}}
\in
\mathbb{R}^{N_q \times d}.
\label{eq:expanded_action_queries}
\end{equation}
The expansion preserves the pretrained register-token prior while providing
more query slots for organizing action-chunk generation.

After initialization, the expanded action-query tensor is used to initialize the
learnable action-query parameter:
\begin{equation}
E_{\mathrm{aq}}
\leftarrow
Q^{\mathrm{act}},
\qquad
E_{\mathrm{aq}}
\in
\mathbb{R}^{N_q \times d}.
\label{eq:eaq_initialization}
\end{equation}
At decoder layer \(\ell\), this parameter forms the action-query memory stream
in TARS:
\begin{equation}
M^{\mathrm{aq}}_{\ell}
=
E_{\mathrm{aq}}.
\label{eq:action_query_memory_stream}
\end{equation}
The current action tokens \(X_\ell\) produce the attention queries, while the
action-query stream provides its own keys and values:
\begin{equation}
Q^a_{\ell}
=
\operatorname{LN}(X_{\ell}) W_q,
\qquad
K^{\mathrm{aq}}_{\ell}
=
\operatorname{LN}(M^{\mathrm{aq}}_{\ell}) W^{\mathrm{aq}}_k,
\qquad
V^{\mathrm{aq}}_{\ell}
=
\operatorname{LN}(M^{\mathrm{aq}}_{\ell}) W^{\mathrm{aq}}_v .
\label{eq:action_query_kv}
\end{equation}
The action-query attention output is then computed as one stream-specific TARS
output:
\begin{equation}
A^{\mathrm{aq}}_{\ell}
=
\operatorname{Softmax}
\left(
\frac{
Q^a_{\ell}
\left(K^{\mathrm{aq}}_{\ell}\right)^{\top}
}{
\sqrt{d_h}
}
\right)
V^{\mathrm{aq}}_{\ell},
\label{eq:action_query_attention_output}
\end{equation}
where \(d_h\) is the attention head dimension.

Together with the self, dense, and task streams, the action-query stream is
aggregated by the stream-separated attention rule:
\begin{equation}
\widehat{X}_{\ell}
=
X_{\ell}
+
\left(
\sum_{b\in\mathcal{B}}
\gamma_b A^b_{\ell}
\right) W_o,
\qquad
\mathcal{B}
=
\{\mathrm{self},\mathrm{aq},\mathrm{dense},\mathrm{task}\}.
\label{eq:tars_stream_aggregation_action_query}
\end{equation}
Here, \(\gamma_{\mathrm{self}}=\gamma_{\mathrm{task}}=1\), while
\(\gamma_{\mathrm{aq}}\) and \(\gamma_{\mathrm{dense}}\) are learned gates, as
defined in the main TARS formulation. The dense stream
\(M^{\mathrm{dense}}_{\ell}=[D_\ell;S]\) and the task stream
\(M^{\mathrm{task}}_{\ell}=\{r\}\) remain unchanged.

Thus, the action-query stream introduces a lightweight pretrained
register-token prior into the action decoder without changing the external
visual, proprioceptive, or task-conditioning streams. In the fixed-query
ablation, \(E_{\mathrm{aq}}\) is kept frozen after initialization. In the default
setting, it is optimized jointly with the rest of the policy.

\subsection{Stream-Separated Normalization in TARS}
\label{app:tars_stream_normalization}

Unlike the shared-softmax Bridge Attention used in VLA-Adapter~\cite{wang2026vlaadapter},
a central design choice in \textsc{TARS} is that heterogeneous conditioning
streams are not normalized within a single shared attention distribution.
Instead, \textsc{TARS} performs stream-wise normalization and then aggregates the
resulting stream outputs. This design is motivated by the fact that the
conditioning memories in our decoder are heterogeneous both semantically and
structurally: the self stream contains action tokens, the action query stream, the dense stream contains many patches and proprioceptive tokens, and the task stream contains a compact task token. A
shared softmax over all these tokens would therefore entangle two distinct
questions: which token is important within a stream, and which stream should
contribute more to the action update.

We formalize this distinction for a single query. Let \(\mathcal B\) denote the
set of conditioning streams. Stream \(b\in\mathcal B\) contains \(M_b\)
key-value tokens, with attention score \(s_{b,j}\) and value \(v_{b,j}\) for
\(j=1,\ldots,M_b\). A shared-softmax attention first pools all tokens from all
streams into one normalization set:
\begin{equation}
    \alpha_{b,j}
    =
    \frac{\exp(s_{b,j})}
    {\sum_{b' \in \mathcal{B}} \sum_{m=1}^{M_{b'}} \exp(s_{b',m})},
    \qquad
    o_{\mathrm{shared}}
    =
    \sum_{b \in \mathcal{B}} \sum_{j=1}^{M_b} \alpha_{b,j} v_{b,j}.
    \label{eq:shared_softmax}
\end{equation}
In this formulation, tokens from all streams directly compete for the same
probability mass. Consequently, the total contribution assigned to a stream is
implicitly affected by its token count. In the degenerate case
\(s_{b,j}=c\) for all \(b,j\), the total attention mass of stream \(b\) becomes
\begin{equation}
    \sum_{j=1}^{M_b} \alpha_{b,j}
    =
    \frac{M_b}{\sum_{b' \in \mathcal{B}} M_{b'}} .
    \label{eq:length_bias}
\end{equation}
Thus, a stream with more tokens receives a larger total mass even when its individual tokens are not more informative. This induces a structural length bias: dense visual patch streams can dominate compact streams such as task or action-query tokens simply because they contain more candidates.

\textsc{TARS} avoids this bias by normalizing each stream independently:
\begin{equation}
    \widetilde{\alpha}_{b,j}
    =
    \frac{\exp(s_{b,j})}
    {\sum_{m=1}^{M_b} \exp(s_{b,m})},
    \qquad
    o_b
    =
    \sum_{j=1}^{M_b} \widetilde{\alpha}_{b,j} v_{b,j}.
    \label{eq:split_softmax}
\end{equation}
This guarantees
\begin{equation}
    \sum_{j=1}^{M_b} \widetilde{\alpha}_{b,j} = 1,
    \qquad \forall b \in \mathcal{B},
    \label{eq:stream_mass}
\end{equation}
independent of the number of tokens in the stream. The stream-wise softmax
therefore assigns probability mass only within each representation type. The
relative strength of different streams is then controlled explicitly by
aggregation coefficients rather than implicitly by representation length:
\begin{equation}
    o_{\mathrm{TARS}}
    =
    \sum_{b \in \mathcal{B}} \gamma_b o_b,
    \label{eq:tars_aggregation}
\end{equation}
where \(\gamma_b\) is either fixed or learned. In our implementation,
\(\gamma_{\mathrm{self}}=\gamma_{\mathrm{task}}=1\), while
\(\gamma_{\mathrm{lat}}\) and \(\gamma_{\mathrm{dense}}\) are learned gates.

This separation has three practical benefits. First, it removes the structural
advantage of long memory streams, making the decoder more robust to changes in
the number of visual patches, observation steps, or auxiliary tokens. Second, it
protects compact but semantically important conditions, such as task identifiers,
from being diluted by dense visual memories. Third, it improves interpretability:
the stream-wise attention weights describe which tokens are selected within each
source, while the gates describe how strongly each source contributes to the
action update. In this sense, \textsc{TARS} treats heterogeneous streams as
parallel information providers rather than forcing them to compete inside a
single attention pool.

\subsection{Flow-Consistent Inpainting and Real-Time Equivalent Execution Analysis}
\label{app:real_time_equivalent_analysis}

We provide the detailed update rule of flow-consistent inpainting and analyze
its exposed stall under the real-time equivalent execution metric.

\paragraph{Flow-consistent overlap constraint.}
Assume that the current action chunk has \(K_{\mathrm{ov}}\) actions remaining.
These actions define the overlap segment
\begin{equation}
Y_{\mathrm{ov}}
\in
\mathbb{R}^{K_{\mathrm{ov}}\times d_a}.
\label{eq:app_overlap_segment}
\end{equation}
The next chunk should start with this segment so that the executed action stream
remains continuous when the controller switches from the current chunk to the
new chunk.

A naive solution is to directly clamp the first \(K_{\mathrm{ov}}\) actions of
the new chunk to \(Y_{\mathrm{ov}}\) at every flow step. However, this is
inconsistent with the rectified-flow training distribution. During training, the
intermediate action state follows
\begin{equation}
A_{\tau}
=
(1-\tau)\epsilon
+
\tau A,
\label{eq:app_rf_interpolation}
\end{equation}
where \(\epsilon\) is the initial noise and \(A\) is the target action chunk.
Therefore, at early flow times, the overlap region should still be close to
noise rather than being directly fixed to the final action value. Hard clamping
would query the decoder on out-of-distribution inputs and may also disturb the
generation of the remaining free actions through self-attention.

Flow-consistent inpainting avoids this issue by constraining the overlap region
to follow the same noise-to-data interpolation as in training. For each new
chunk, we sample a fixed overlap noise
\begin{equation}
\epsilon_{\mathrm{ov}}
\sim
\mathcal{U}[-1,1]^{K_{\mathrm{ov}}\times d_a}.
\label{eq:app_overlap_noise}
\end{equation}
Let
\begin{equation}
\tau_i
=
\frac{i}{N_{\mathrm{FE}}},
\qquad
i=0,\ldots,N_{\mathrm{FE}} .
\label{eq:app_flow_time}
\end{equation}
At each flow time, the overlap region is constrained as
\begin{equation}
A^{(i)}_{1:K_{\mathrm{ov}}}
=
(1-\tau_i)\epsilon_{\mathrm{ov}}
+
\tau_i Y_{\mathrm{ov}} .
\label{eq:app_flow_consistent_overlap}
\end{equation}
For
\begin{equation}
i=0,\ldots,N_{\mathrm{FE}}-1,
\label{eq:app_flow_step_range}
\end{equation}
the full action chunk is first updated by the standard rectified-flow step:
\begin{equation}
\widetilde{A}^{(i+1)}
=
A^{(i)}
+
\frac{1}{N_{\mathrm{FE}}}
v_{\theta}
\bigl(
A^{(i)},\tau_i,\mathcal{C}
\bigr).
\label{eq:app_rf_update}
\end{equation}
The free region is kept from this update, while the overlap region is projected
back to the flow-consistent path:
\begin{equation}
A^{(i+1)}_{1:K_{\mathrm{ov}}}
=
(1-\tau_{i+1})\epsilon_{\mathrm{ov}}
+
\tau_{i+1}Y_{\mathrm{ov}},
\label{eq:app_overlap_projection_fixed}
\end{equation}
\begin{equation}
A^{(i+1)}_{K_{\mathrm{ov}}+1:H}
=
\widetilde{A}^{(i+1)}_{K_{\mathrm{ov}}+1:H}.
\label{eq:app_overlap_projection_free}
\end{equation}
By construction, the final overlap satisfies
\begin{equation}
A^{(N_{\mathrm{FE}})}_{1:K_{\mathrm{ov}}}
=
Y_{\mathrm{ov}},
\label{eq:app_final_overlap}
\end{equation}
which guarantees boundary consistency. At the same time, for every intermediate
flow time, the overlap region remains on the same linear interpolation path as
the rectified-flow training distribution. This is the key difference between
flow-consistent inpainting and hard inpainting: hard inpainting enforces only
the final boundary value, whereas flow-consistent inpainting preserves both
boundary continuity and in-distribution decoder inputs. The computational
overhead of Eq.~\eqref{eq:app_overlap_projection_fixed} is only an element-wise
update over \(K_{\mathrm{ov}}\times d_a\) entries per flow step, which is
negligible compared with evaluating the velocity network \(v_{\theta}\).

\paragraph{Exposed stall under real-time equivalent execution.}
Recall the real-time equivalent execution metric:
\begin{equation}
T
=
N_{\mathrm{actions}}\Delta t
+
S_{\mathrm{method}},
\label{eq:app_tmethod}
\end{equation}
where \(N_{\mathrm{actions}}\Delta t\) is the nominal physical execution time
under the target control period \(\Delta t\), and \(S_{\mathrm{method}}\) is the
inference latency that remains exposed on the control path.

For a synchronous chunk-based policy, each inference call must finish before the
corresponding chunk can be executed. If the policy performs \(n\) inference
calls in one episode and each call takes latency \(L\), the exposed stall is
\begin{equation}
S_{\mathrm{method}}^{\mathrm{sync}}
=
nL .
\label{eq:app_stall_sync}
\end{equation}

For asynchronous flow-consistent inpainting, only the first chunk requires a
cold-start inference before execution starts. After that, each new inference is
launched while the remaining \(K_{\mathrm{ov}}\) actions of the current chunk
are still being executed. These \(K_{\mathrm{ov}}\) actions provide an execution
window of length \(K_{\mathrm{ov}}\Delta t\). Therefore, only the part of the
inference latency that exceeds this window remains exposed:
\begin{equation}
S_{\mathrm{method}}^{\mathrm{FCI}}
=
L_0
+
(n-1)\max(0,L-K_{\mathrm{ov}}\Delta t),
\label{eq:app_stall_async}
\end{equation}
where \(L_0\) is the cold-start latency, \(L\) is the per-chunk inference
latency after execution starts, \(K_{\mathrm{ov}}\) is the overlap length,
\(\Delta t\) is the target control period, and \(n\) is the number of inference
calls in the episode.

Eq.~\eqref{eq:app_stall_async} shows that the asynchronous stall collapses to
the cold-start term whenever
\begin{equation}
L
\leq
K_{\mathrm{ov}}\Delta t.
\label{eq:app_hidden_condition}
\end{equation}
Equivalently, the critical control period is
\begin{equation}
\Delta t^{\star}
=
\frac{L}{K_{\mathrm{ov}}}.
\label{eq:app_critical_dt}
\end{equation}
For any target control period satisfying
\begin{equation}
\Delta t
\geq
\Delta t^{\star},
\label{eq:app_control_period_condition}
\end{equation}
all inference calls after the first chunk can be hidden inside action execution,
and the effective episode duration approaches
\begin{equation}
T
\approx
N_{\mathrm{actions}}\Delta t
+
L_0 .
\label{eq:app_async_episode_time}
\end{equation}

\paragraph{Implementation-specific latency.}
In our implementation on the workspace with RTX 4080, the average rectified-flow inference latency is
\begin{equation}
L
=
56.32\,\mathrm{ms}.
\label{eq:app_measured_latency}
\end{equation}
With an overlap length of
\begin{equation}
K_{\mathrm{ov}}
=
4,
\label{eq:app_overlap_length}
\end{equation}
the critical control period is
\begin{equation}
\Delta t^{\star}
=
\frac{56.32\,\mathrm{ms}}{4}
=
14.08\,\mathrm{ms},
\label{eq:app_measured_critical_dt}
\end{equation}
which corresponds to approximately \(71.0\,\mathrm{Hz}\). Therefore, for a
target control rate of \(20\,\mathrm{Hz}\), i.e.,
\begin{equation}
\Delta t
=
50\,\mathrm{ms},
\label{eq:app_target_control_period}
\end{equation}
the condition \(L\leq K_{\mathrm{ov}}\Delta t\) is satisfied. The inference
latency after the first chunk is fully hidden by action execution, and the
exposed stall becomes
\begin{equation}
S_{\mathrm{method}}^{\mathrm{FCI}}
=
L_0 .
\label{eq:app_hidden_stall}
\end{equation}
Using the same average inference latency for the cold start gives
\begin{equation}
L_0
=
56.32\,\mathrm{ms}.
\label{eq:app_cold_start_latency}
\end{equation}
Thus, the effective episode duration approaches
\begin{equation}
T
\approx
N_{\mathrm{actions}}\Delta t
+
56.32\,\mathrm{ms},
\label{eq:app_final_episode_time}
\end{equation}
instead of accumulating inference latency at every chunk boundary as in
synchronous execution.


\textbf{Properties of FCI}. This design has two important properties. First, it guarantees the boundary
consistency: when \(i=N_{\mathrm{FE}}\), we have \(\tau_i=1\), and
Eq.~\eqref{eq:flow_consistent_overlap} gives
\(A^{(N_{\mathrm{FE}})}_{1:K_{\mathrm{ov}}}=Y_{\mathrm{ov}}\), so the new chunk exactly matches the remaining part of the current chunk. Second, it keeps the decoder input in
distribution: throughout the flow trajectory, the constrained overlap region
follows the same linear noise-to-data path as in training. Under the metric in
Eq.~\eqref{eq:tmethod-v2}, synchronous execution accumulates inference latency
at every chunk boundary, whereas FCI can hide subsequent inference calls whenever
the overlap execution window is long enough.

\section{Experimental Details}
\label{app:apdxc}

\subsection{Ablation on Action Queries in LIBERO Benchmark}
\label{app:aq_ablation_libero}

We ablate the number and learnability of action queries in the TARS decoder, as shown in Table~\ref{tab:aq_ablation_libero}. Action queries serve as latent conditioning slots for action-chunk generation, and their number and adaptability affect how the decoder refines future actions from visual-proprioceptive evidence.

Using only 4 learnable action queries achieves an average success rate of 73.2\%, indicating insufficient conditioning capacity. Increasing the number to 32 gives the best average performance of 82.1\%, with a particularly strong improvement on LIBERO-Object. When the 32 queries are fixed rather than learnable, the average success rate drops to 77.2\%, suggesting that learnability helps adapt these slots to task- and object-dependent action structure. Further increasing the number to 64 reduces the average success rate to 73.0\%, implying that excessive queries may introduce redundant slots and optimization difficulty.

\begin{table*}[htb]
\centering
\caption{Ablation results of action-query number and learnability in the TARS decoder on the LIBERO benchmark.}
\label{tab:aq_ablation_libero}
\setlength{\tabcolsep}{5pt}
\renewcommand{\arraystretch}{1.08}

\begin{tabular}{lcc|cccc|c}
\toprule
\multirow{2}{*}{\textbf{Arch.}} 
& \multicolumn{2}{c|}{\textbf{Action Query}}
& \multirow{2}{*}{\textbf{Spatial}} 
& \multirow{2}{*}{\textbf{Object}} 
& \multirow{2}{*}{\textbf{Goal}} 
& \multirow{2}{*}{\textbf{Long}} 
& \multirow{2}{*}{\textbf{Avg.}} \\
\cmidrule(lr){2-3}
& \textbf{Num.}
& \textbf{Learnable}
& 
& 
& 
& 
& \\
\midrule

TARS
& 4 & \cmark
& \textbf{90.0} & 66.0 & 79.6 & 57.0 & 73.2 \\

\rowcolor{tablerowgray}
TARS
& 32 & \cmark
& 89.0 & \textbf{94.6} & 78.4 & 66.2 & \textbf{82.1} \\

TARS
& 32 & \xmark
& 84.0 & 73.6 & \textbf{82.2} & \textbf{69.0} & 77.2 \\

TARS
& 64 & \cmark
& 87.8 & 75.2 & 74.4 & 54.4 & 73.0 \\

\bottomrule
\end{tabular}
\end{table*}

We also observe that the learned gate associated with the action-query pathway remains small during training, typically around 0.02. This indicates that action queries do not dominate the decoder, but instead act as a weak yet useful action-refinement condition. The main action generation is still driven by visual, proprioceptive, task, and flow-state information, while learnable queries provide a small adaptive bias for organizing action slots. A possible interpretation is that these queries play a role analogous to register tokens in vision transformers, which have been shown to absorb artifact-like or non-local computational components and thereby protect patch tokens from being repurposed for such internal computations~\citep{darcet2024vision}. In this sense, learnable action queries may help buffer nuisance variation in the conditioning stream and allow dense visual tokens to preserve cleaner action-relevant spatial information.

Overall, the ablation shows that action-query design is a nontrivial factor in TARS. Too few queries limit refinement capacity, fixed queries reduce adaptability, and too many queries can hurt optimization. We therefore use 32 learnable action queries as the default configuration in CT-VAM.

\subsection{Scaling Behavior of Flow-matching Policy in  LIBERO Benchmark}
\label{app:scaling_libero}

We study the scaling behavior of flow-matching policies with different overall model sizes, as shown in Table~\ref{tab:fm_libero_results}. The three variants, FM-S, FM-B, and FM-L, increase both trainable and total parameter counts, allowing us to examine whether larger model capacity alone leads to better visuomotor performance under the same LIBERO evaluation protocol. These flow-matching baselines use the same dual-view visual encoder and rectified-flow action head as CT-VAM, but do not introduce the proposed TARS module.

\begin{table*}[htb]
\centering
\caption{Scaling behavior of flow-matching policies with different overall model sizes on the LIBERO benchmark. The three variants increase both trainable and total parameter counts while following the same evaluation protocol.}
\label{tab:fm_libero_results}
\setlength{\tabcolsep}{5pt}
\renewcommand{\arraystretch}{1.08}

\begin{tabular}{lc|cccc|c}
\toprule
\multicolumn{1}{c}{\textbf{LIBERO}} 
& \textbf{Trainable / Total Params} 
& \textbf{Spatial} 
& \textbf{Object} 
& \textbf{Goal} 
& \textbf{Long} 
& \textbf{Avg.} \\
\midrule

FM-S
& 35M / 68M
& 87.0 & 64.2 & 59.6 & 52.4 & 65.8 \\

FM-B 
& 141M / 227M 
& \textbf{89.2} & 93.0 & \textbf{71.0} & 54.2 & 76.9 \\

FM-L 
& 455M / 758M 
& \textbf{89.2} & \textbf{93.4} & 70.6 & \textbf{56.0} & \textbf{77.3} \\

\bottomrule
\end{tabular}
\end{table*}

Scaling from FM-S to FM-B improves the average success rate from 65.8\% to 76.9\%, indicating that increasing model capacity can substantially benefit flow-matching visuomotor policies when moving from a small to a base-size model. The largest improvement appears on LIBERO-Object, suggesting that object-centric manipulation benefits from stronger visual and action-generation capacity. However, further scaling from FM-B to FM-L only increases the average success rate from 76.9\% to 77.3\%, despite a much larger parameter count. This indicates a clear diminishing-return effect: after a certain capacity, simply enlarging the model provides only marginal gains on LIBERO.

Notably, CT-VAM achieves 82.1\% average success with only 68M parameters, outperforming the largest FM-L variant with 758M total parameters. This result suggests that the performance of CT-VAM is not explained by model size alone. Instead, the proposed TARS stream routing and learnable action-query design provide a more effective way to use visual, proprioceptive, task, and action information for closed-loop visuomotor control. These results support our design choice of pursuing a compact but structured vision-action model rather than relying solely on brute-force scaling.

\begin{figure}[h]
    \centering
\includegraphics[width=1\linewidth]{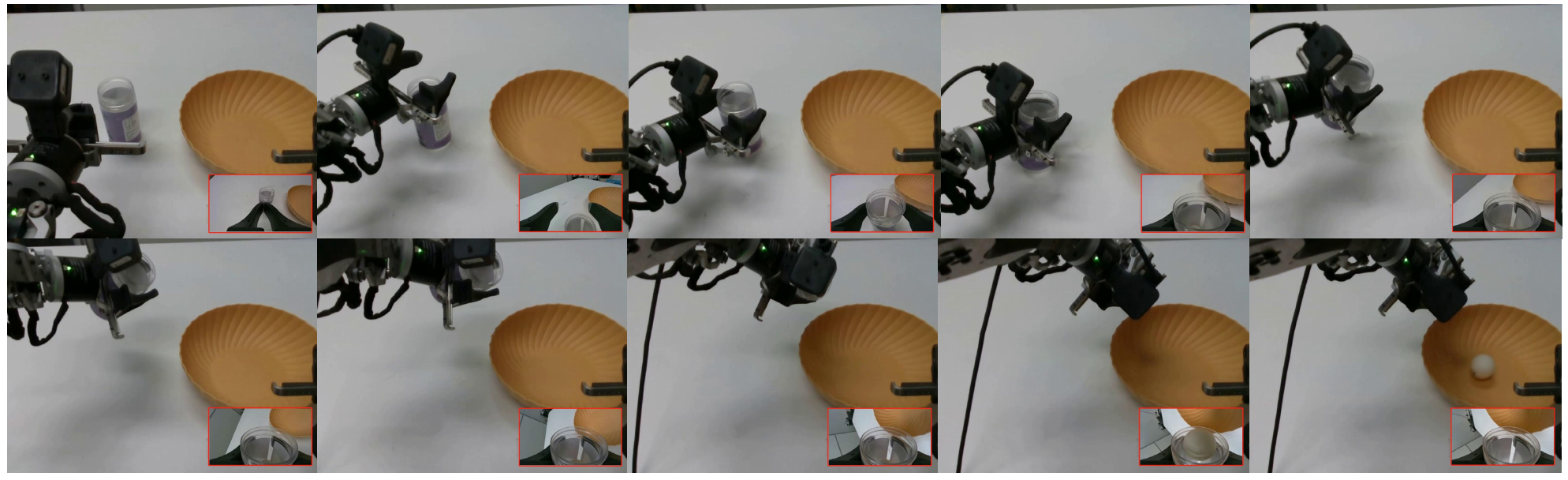}
    \caption{Key frames of the ball pouring task. The robot is required to grasp a bottle from the tabletop, adjust its pose, and pour the contained balls into the target tray. This task evaluates whether CT-VAM can perform stable grasping and fine-grained pouring control under visual feedback. The red inset shows the wrist-camera observation during execution.}
    \label{fig:pour-task-frame}
\end{figure}

\subsection{Real-World Experiments}
\label{app:real_world_exper}

\textbf{Hardware config}. Real-world experiments are conducted on the OpenArm platform. We use one of its
7-DoF arms equipped with a gripper for the evaluated tabletop manipulation
tasks. The perception setup includes an Intel RealSense D435 camera as the
third-person camera and an Intel RealSense D405 camera mounted on the wrist of
the active arm. The D435 provides the external RGB observation, while the D405
provides the wrist-view RGB observation. Demonstrations are collected using a PICO 4 Ultra teleoperation device. For
real-robot execution, experiments are conducted on a workstation equipped with
an NVIDIA RTX 4080. In addition, we evaluate the inference efficiency of CT-VAM
on an NVIDIA Jetson Orin NX platform to assess its edge deployment capability.

\textbf{Data Collection and Training Protocol.}
For each task, we collect 30 demonstration episodes using a PICO 4 Ultra teleoperation device. Each episode is recorded at 30 Hz and contains synchronized RGB observations from both the D435 third-person camera ($640 \times 480$) and the D405 wrist-mounted camera ($424 \times 240$), along with the 7-DoF joint positions and a continuous gripper state.
The action space consists of 7-dimensional joint position targets and a continuous 1-dimensional gripper command, yielding $d_a = 8$ per step. The average episode length is approximately 236 steps for the ball pouring task and 433 steps for the box opening and placement task. All models are trained with a batch size of 32 for 100000 steps.

\textbf{Evaluation Protocol.} Each method--platform configuration is evaluated over 20 independent trials. At the start of each trial, the bottle, the box, or the object is placed at a random position within the robot's reachable tabletop workspace, introducing variation in the initial conditions across trials. A trial is considered successful only if the robot completes the full task without human intervention. For the box opening and placement task, the one-hot task label is switched manually by an operator when the current subtask is visually judged to be complete. No automatic subtask recognition module is used. A trial is counted as successful only if all three subtasks are completed consecutively.

\begin{figure}[htb]
    \centering
    \includegraphics[width=1\linewidth]{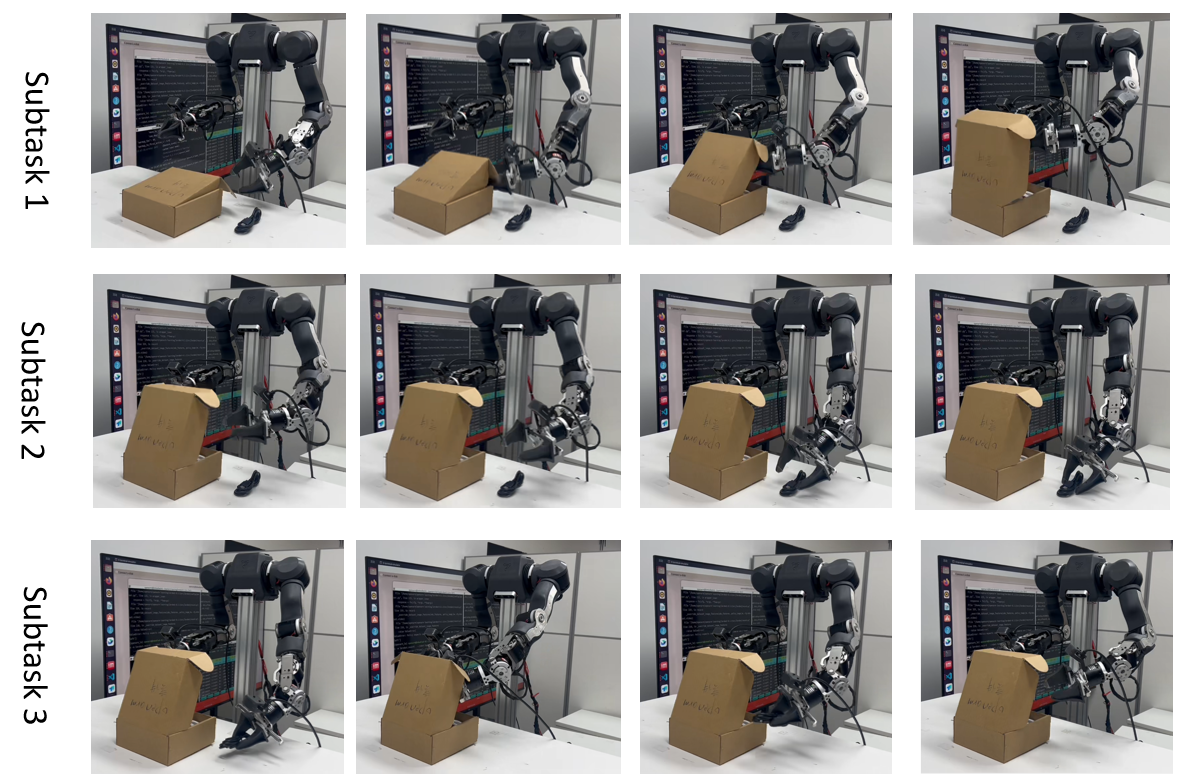}
    \caption{Key frames of the long-horizon box opening and placement task. The task is divided into three manually labeled subtasks during data collection: Subtask 1 opens the box, Subtask 2 grasps the target object, and Subtask 3 places the object into the box. During evaluation, the corresponding one-hot task label is manually switched at subtask boundaries to test whether CT-VAM can execute consecutive subtasks under task-conditioned control.}
    \label{fig:task2}
\end{figure}

\textbf{Ball Pouring: Result Analysis.} Figure~\ref{fig:pour-task-frame} shows representative key frames of a successful ball pouring trial. Across trials, CT-VAM reliably localizes the bottle under varying initial placements, executes stable grasping, and completes the pouring motion without manual intervention in most cases. We observe that the only failure case of CT-VAM with FCI occurs during bottle grasping, where a slightly tilted grasp pose leads to an unstable grasp before pouring. A similar failure mode is also observed for $\pi_0$, suggesting that this error mainly comes from the sensitivity of the bottle grasping stage to small pose deviations, rather than from the flow-consistent inpainting mechanism itself. In contrast, diffusion policy more often fails due to timeout, where the robot makes little or no progress after task initialization. We attribute this behavior to insufficiently consistent and goal-directed action predictions under real-world visual perturbations. In such cases, the predicted joint updates tend to be too small or unstable to drive effective motion through the low-level controller, causing the execution to stall until timeout.

\textbf{Box Opening and Placement.}
We further evaluate CT-VAM on a long-horizon box opening and placement task to
examine whether the policy can execute multiple consecutive subtasks under
task-condition switching. During data collection, each demonstration is
segmented into subtask-level trajectories and assigned a corresponding task
label, which is encoded using the same one-hot task representation as in the
main policy. During evaluation, we manually switch the task label at subtask
boundaries, rather than using an automatic subtask recognition module, because
subtask segmentation and online task inference are outside the scope of this
paper.


The policy is required to sequentially open the box, grasp the target object,
and place it into the box within a single execution episode. This task
therefore evaluates whether the task-conditioning branch can effectively guide
action generation through \textsc{TARS}, and whether the proposed
stream-separated routing mechanism can balance task information with visual and
proprioceptive observations during long-horizon execution. Figure~\ref{fig:task2}
shows the key frames of a representative successful trial. In particular, when
the task condition is switched at a subtask boundary, CT-VAM is able to redirect
the policy toward the next subtask while preserving visually grounded closed-loop
control. These results suggest that \textsc{TARS} can effectively coordinate
task-level intent with visual and proprioceptive evidence, making CT-VAM
compatible with long-horizon manipulation tasks that require sequential
task-conditioned execution.


\subsection{Jetson Orin NX Deployment and TensorRT Optimization}
\label{app:jetson_trt_deployment}

For edge deployment, we run CT-VAM fully onboard on an NVIDIA Jetson Orin NX platform. This platform is designed for embedded AI and robotics applications with constrained power, memory, and compute budgets~\cite{nvidia_jetson_orin}. Since Jetson Orin NX has more limited compute and memory resources than a desktop GPU, directly executing the PyTorch model introduces non-negligible runtime overhead. To better match the intended edge-deployment setting, we convert the trained CT-VAM checkpoint into TensorRT engines for accelerated on-device inference. TensorRT is NVIDIA's inference optimization and runtime SDK for high-performance deep neural network deployment on NVIDIA GPUs, supporting model import from frameworks and intermediate formats such as PyTorch and ONNX~\cite{nvidia_tensorrt,onnx}.


The TensorRT conversion only changes the inference backend. It does not modify
the model architecture, training data, learned parameters, observation-action
space, action horizon, number of flow inference steps, or low-level OpenArm
controller. The same trained checkpoint is used for both desktop-GPU and Jetson
deployment. In our implementation, the PyTorch checkpoint
\texttt{step\_00100000.pt}, trained for 100k steps, is exported through the
PyTorch $\rightarrow$ ONNX $\rightarrow$ TensorRT pipeline. To simplify runtime
execution, the policy is decomposed into two TensorRT subgraphs: a
vision--proprioception encoder and a single flow-matching denoiser step. During
one action-chunk generation, the encoder is executed once, and the denoiser-step
engine is repeatedly invoked for the prescribed number of flow inference steps.

We use TensorRT mixed precision with BF16 enabled and FP32 fallback for layers
that are not supported in BF16. The Jetson software stack is JetPack 6.2
(L4T R36.4.3), CUDA 12.6, TensorRT 10.3.0, and cuDNN 9.3.0. The TensorRT engines
are built with batch size 1, static input shapes, and a 4 GiB workspace. All
latency numbers are measured after runtime warm-up. The reported inference time
corresponds to one complete action-chunk generation, including the
vision--proprioception encoder and all flow-matching denoising steps. It excludes
camera image acquisition, robot communication, and low-level controller
execution, which are not part of the neural policy inference time.

We also verify that the TensorRT backend preserves the numerical behavior of the
policy. On 50 observation samples drawn from the evaluation dataset, the
difference between the PyTorch and TensorRT action outputs is
$3.9\times10^{-3}$ rad in mean absolute error, corresponding to $0.26\%$ of the
mean per-joint action range. The maximum absolute error is
$5.9\times10^{-2}$ rad. This comparison is performed in joint-angle space with
the flow-matching noise fixed across the two backends. These results indicate
that the TensorRT conversion primarily improves deployment efficiency rather
than changing the learned control policy.

The reduced inference latency is particularly important for the Ball Pouring
task. This task requires the robot to grasp the bottle, adjust its pose, and
pour small balls into the target tray under continuous visual feedback. Excessive
inference latency can lead to stale observations, delayed corrective actions,
and less stable pouring behavior. Therefore, improving the onboard inference
speed not only reduces execution time but can also improve the stability of
closed-loop manipulation under the same evaluation protocol.

\subsection{Real-Time-Equivalent Execution Frequency}
\label{app:real_time_equivalent_frequency}

We provide a concise calculation of the real-time-equivalent execution frequency
under FCI. The detailed exposed-stall formulation has been introduced in
Appendix~\ref{app:real_time_equivalent_analysis}; here we only convert it into an
effective low-level action frequency for the real-world deployment setting.

Let the low-level controller run at a target control period $\Delta t$. In our
real-world experiments, the target control frequency is $20~\mathrm{Hz}$, so
\begin{equation}
    \Delta t = 50~\mathrm{ms}.
\end{equation}
CT-VAM predicts action chunks of length
\begin{equation}
    H = 8,
\end{equation}
and FCI uses an overlap length of
\begin{equation}
    K_{\mathrm{ov}} = 4.
\end{equation}

The overlap actions are not executed twice. They are used to enforce boundary
consistency between the tail of the current chunk and the beginning of the next
chunk. When the new chunk becomes active, its first $K_{\mathrm{ov}}$ actions
correspond to the overlap segment and are not counted as newly executed actions.
Therefore, each subsequent FCI chunk advances the executed action sequence by
\begin{equation}
    H_{\mathrm{new}}
    =
    H - K_{\mathrm{ov}}
    =
    4
\end{equation}
new low-level actions. The physical time required to execute these new actions
is
\begin{equation}
    H_{\mathrm{new}}\Delta t
    =
    4 \times 50~\mathrm{ms}
    =
    200~\mathrm{ms}.
\end{equation}
The overlap execution window available for hiding the next inference call is
\begin{equation}
    K_{\mathrm{ov}}\Delta t
    =
    4 \times 50~\mathrm{ms}
    =
    200~\mathrm{ms}.
\end{equation}

Let $L$ denote the measured per-chunk inference latency after execution has
started. The latency that remains exposed on the control path is
\begin{equation}
    L_{\mathrm{exposed}}
    =
    \max\left(0, L - K_{\mathrm{ov}}\Delta t\right).
\end{equation}
Accordingly, the steady-state real-time-equivalent execution frequency under FCI
is
\begin{equation}
    f_{\mathrm{eq}}^{\mathrm{FCI}}
    =
    \frac{H_{\mathrm{new}}}
    {
    H_{\mathrm{new}}\Delta t
    +
    \max\left(0, L - K_{\mathrm{ov}}\Delta t\right)
    }.
\label{eq:fci_equivalent_frequency}
\end{equation}
This expression counts only the newly executed actions and therefore avoids
double-counting the overlap region.

On the RTX 4080 platform, the measured inference latency of CT-VAM with FCI is
\begin{equation}
    L_{\mathrm{RTX}} = 56.84~\mathrm{ms}.
\end{equation}
Since $L_{\mathrm{RTX}} < K_{\mathrm{ov}}\Delta t$, the exposed latency is
\begin{equation}
    L_{\mathrm{exposed}}^{\mathrm{RTX}}
    =
    \max\left(0, 56.84~\mathrm{ms} - 200~\mathrm{ms}\right)
    =
    0.
\end{equation}
Therefore, the real-time-equivalent execution frequency is
\begin{equation}
    f_{\mathrm{eq,RTX}}^{\mathrm{FCI}}
    =
    \frac{4}{200~\mathrm{ms}}
    =
    20.0~\mathrm{Hz}.
\end{equation}

On Jetson Orin NX with the TensorRT-optimized backend, the measured inference
latency of CT-VAM with FCI is
\begin{equation}
    L_{\mathrm{Jetson}} = 200.60~\mathrm{ms}.
\end{equation}
The exposed latency is
\begin{equation}
    L_{\mathrm{exposed}}^{\mathrm{Jetson}}
    =
    \max\left(0, 200.60~\mathrm{ms} - 200~\mathrm{ms}\right)
    =
    0.60~\mathrm{ms}.
\end{equation}
Thus, the real-time-equivalent execution frequency is
\begin{equation}
    f_{\mathrm{eq,Jetson}}^{\mathrm{FCI}}
    =
    \frac{4}{200~\mathrm{ms} + 0.60~\mathrm{ms}}
    \approx
    19.94~\mathrm{Hz}.
\end{equation}

These results show that FCI hides most of the neural inference latency inside
physical action execution. On RTX 4080, the inference latency is fully hidden in
the steady state, and the effective execution frequency reaches the target
$20~\mathrm{Hz}$ low-level control rate. On Jetson Orin NX, only
$0.60~\mathrm{ms}$ remains exposed per FCI stride, so the effective execution
frequency remains close to $20~\mathrm{Hz}$.

\end{document}